\documentclass[12pt]{article}
\usepackage{arxiv}
\usepackage{color}
\usepackage[tbtags]{amsmath}
\usepackage{mathtools,layout,graphicx,framed,xpatch,multirow,tabularx,gensymb,amssymb,enumitem,float,bbm,algorithm,algorithmic,adjustbox,booktabs,amsthm}
\usepackage[strings]{underscore} 
\DeclareMathOperator*{\argmax}{arg\!\max}

\newtheorem{theorem}{Theorem}

\newtheorem{lemma}{Lemma}

\newtheorem{definition}{Definition}

\newcommand{\ignore}[1]{}

\begin{document}
\title{Maximizing Submodular or Monotone Functions under Partition Matroid Constraints by Multi-objective Evolutionary Algorithms}
\shorttitle{Submodular Functions under Partition Matroid Constraints by MOEAs}
%

\author{Anh Viet Do
\\Optimisation and Logistics
\\The University of Adelaide, Adelaide, Australia
\And
Frank Neumann
\\Optimisation and Logistics
\\The University of Adelaide, Adelaide, Australia
}

\maketitle              
\begin{abstract}
Many important problems can be regarded as maximizing submodular functions under some constraints. A simple multi-objective evolutionary algorithm called GSEMO has been shown to achieve good approximation for submodular functions efficiently. While there have been many studies on the subject, most of existing run-time analyses for GSEMO assume a single cardinality constraint. In this work, we extend the theoretical results to partition matroid constraints which generalize cardinality constraints, and show that GSEMO can generally guarantee good approximation performance within polynomial expected run time. Furthermore, we conducted experimental comparison against a baseline GREEDY algorithm in maximizing undirected graph cuts on random graphs, under various partition matroid constraints. The results show GSEMO tends to outperform GREEDY in quadratic run time.

\end{abstract}
\section{Introduction}
The area of runtime analysis has made significant contributions to the theory of evolutionary algorithms over the last 25 years~\cite{Auger11,ncs/Jansen13}. Important results have been obtained for a wide range of benchmark functions as well as for important combinatorial optimization problems~\cite{DBLP:books/daglib/0025643}. This includes a wide range of evolutionary computing methods in a wide range of deterministic, stochastic and dynamic settings. We refer the reader to~\cite{BookDoeNeu} for a presentation of important recent research results. 

Many important real-world problems can be stated in terms of optimizing a submodular function and the analysis of evolutionary algorithms using multi-objective formulations has shown that they obtain in many cases the best possible performance guarantee (unless P=NP). Important recent results on the use of evolutionary algorithms for submodular optimization are summarized in~\cite{DBLP:books/sp/ZhouYQ19}.
The goal of this paper is to expand the investigations of evolutionary multi-objective optimization for submodular optimization. While previous investigations mainly concentrated on monotone submodular functions with a single constraint, we consider non-monotone submodular functions with a set of constraints.

\subsection{Related work}

Submodular functions are considered the discrete counterparts of convex functions \cite{submodconvex}. Submodularity captures the notion of diminishing marginal return, and is present in many important problems. While minimizing submodular functions can be done using a polynomial time combinatorial algorithm \cite{SubmodMin}, submodular maximization encompasses many NP-hard combinatorial problems such as maximum coverage, maximum cut \cite{MaxcutSDP}, maximum influence \cite{MaxSpreadInfluence}, and sensor placement problem \cite{NearSensorPlacement,SubInfoGather}. It is also applied in many problems in machine learning domain \cite{FeatureSelection,DataSubsetSelection,DocSum,DocSumBudget,EfficientMinDecompSubmodular}. Considering the role of evolutionary algorithms in difficult optimization problems, we focus on submodular maximization.

Realistic optimization problems often impose constraints on the solutions. In applications of submodular maximization, Matroid and Knapsack constraints are among the most common \cite{NonMonoSubmodularMaxMatroidKnap}. In this work, we consider submodular maximization under partition matroid constraints, which are a generalization of cardinality constraints. This type of constraint has been considered in a variety of applications \cite{CouplingEdgeGraphCut,MaxCoverGroupBudget,TightApproxAlgoMaxGenAsg}.

A greedy algorithm has been shown to achieve $1/2$-approximation ratio in maximizing monotone submodular functions under partition matroid constraints \cite{LocationBankOptFloat}. It was later proven that $(1-1/e)$ is the best approximation ratio a polynomial time algorithm can guarantee. A more recent study \cite{MaxSubmodularMatroid} proposed a randomized algorithm that achieves this ratio. Another study \cite{DeterministicSubMax} analyzes derandomizing search heuristics, leading to a deterministic $0.5008$-approximation ratio.

Additionally, more nuanced results have been reached when limiting objective functions to those with finite rate of marginal change, quantified by curvature $\alpha$ as defined in \cite{GreedyMaxBoundCurvPartMatroid}. The results in \cite{SubGreedGeneralRado,SubCurvOpt} indicate that $\frac{1}{\alpha}(1-e^{-\alpha})$-approximation ratio is achievable by the continuous greedy algorithm in maximizing monotone submodular functions under a matroid constraint. A more recent study \cite{GreedNonmod} proved $\frac{1}{\alpha}(1-e^{-\gamma\alpha})$-approximation ratio for the deterministic greedy algorithm in maximizing functions with submodularity ratio $\gamma$, under a cardinality constraint.

These results rely on the assumption of monotonicity of the objective functions, $f(S)\leq f(T)$ for all $S\subseteq T$, which do not hold in many applications of submodular maximization. A study \cite{GSEMO2015} derives approximation guarantees for GSEMO algorithm in maximizing monotone and symmetric submodular function under a matroid constraint, which suggest that non-monotone functions are harder to maximize. This is supported by another result \cite{GreedyMaxBoundCurvPartMatroid} for a greedy algorithm in maximizing general submodular function under partition matroid constraints. A recent study \cite{MaxMonoApproxSubmodMulti} extends the results for a GSEMO variant to the problems of maximizing general submodular functions, but under a cardinality constraint.

\subsection{Our contribution}
In this work, we contribute to the theoretical analysis of GSEMO by generalizing previous results to partition matroid constraints. Firstly, we provide an approximation guarantee for GSEMO in maximizing general submodular functions under partition matroid constraints (Theorem \ref{gsemo_submodular}). Secondly, we derive another result for monotone and not necessarily submodular functions, under the same constraints (Theorem \ref{gsemo_monotone}), to account for other important types of function like subadditive functions. Subadditivity encompasses submodularity, and is defined by the property where the whole is no greater than the sum of parts. Subadditive functions are commonly used to model items evaluations and social welfare in combinatorial auctions \cite{ItemPrice,WelfareGuaranteesCombAuc,ComposeEffiMech,GreedyMaxBoundCurvPartMatroid}. Our results extend the existing ones \cite{MaxMonoApproxSubmodMulti} with more refined bounds.

We investigate GSEMO's performance against GREEDY's \cite{GreedyMaxBoundCurvPartMatroid} in maximizing undirected cuts in random graphs under varying cardinality constraints and partition matroid constraints. Graph cut functions with respect to vertices sets are known to be submodular and non-monotone \cite{NonMonoSubMaxCut}. In particular, they are also symmetric for undirected graphs \cite{ComMinSymSub}. Our results suggest that GSEMO typically requires more evaluations to reach GREEDY's outputs quality. Nonetheless, GSEMO surpasses GREEDY shortly after the latter stops improving,
indicating the former's capacity for exploring the search spaces. Predictably, GSEMO outperforms GREEDY more reliably in larger search spaces.

The paper is structured as follows. We formally define the problems and the algorithms in Section~\ref{sec2}. In Section~\ref{sec3}, we analyze GSEMO with respect to its approximation behaviour and runtime and report on our experimental investigations in Section~\ref{sec4}. Finally, we finish with some conclusions.

\section{Preliminaries}
\label{sec2}
In this section, we provide a formal definition of the problem and some of its parameters relevant to our analyses. We also describe the simple GREEDY algorithm and the GSEMO algorithm considered in this work.
\subsection{Problem definition}
We consider the following problem. Let $f:2^V\to\mathbb{R}^+$ be a non-negative function over a set $V$ of size $n$, $B=\{B_i\}_{i=1,\dots,k}$ be a partition of $V$ for some $k\leq n$, $D=\{d_i\}_{i=1,\dots,k}$ be integers such that $d_i\in[1,|B_i|]$ for all $i$, the problem is finding $X\subseteq V$ maximizing $f(X)$, subject to
\[|X\cap B_i|\leq d_i,\quad\forall i=1,\dots,k.\]
These constraints are referred to as partition matroid constraints, which are equivalent to a cardinality constraint if $k=1$. The objective function $f$ of interest is submodular, meaning it satisfies the property as defined in \cite{BestAlgoApproxMaxSubmodular}
\begin{definition}
A function $f:2^V\to\mathbb{R^+}$ is submodular if
\[f(X\cup\{v\})-f(X)\geq f(Y\cup\{v\})-f(Y),\quad\forall X\subseteq Y\subseteq V,v\in V\setminus Y.\]
\end{definition}
We can assume that $f$ is not non-increasing, and for monotone $f$, we can assume $f(\emptyset)=0$. To perform analysis, we define the function's monotonicity approximation term similar to \cite{NearSensorPlacement}, but only for subsets of a certain size
\begin{definition}\label{mono_approx}
For a function $f:2^V\to\mathbb{R}^+$, its monotonicity approximation term with respect to a parameter $j$ is
\[\epsilon_j=\max_{X,v:|X|<j}\{f(X\setminus\{v\})-f(X)\},\]
for $j>0$ and $\epsilon_0=0$.
\end{definition}
It is clear that $\epsilon_j$ is non-negative, non-decreasing with increasing $j$, and $f$ is monotone iff $\epsilon_n=0$. Additionally, for monotone non-submodular $f$, we use submodularity ratio which quantifies how close $f$ is to being modular. In particular, we simplify the definition \cite{spectral} which measures the severity of the diminishing return effect.
\begin{definition}\label{submodratio_def}
For a monotone function $f:2^V\to\mathbb{R}^+$, its submodularity ratio with respect to two parameters $i$, $j\geq1$ is
\[\gamma_{i,j}=\min_{|X|<i,|L|\leq j,X\cap L=\emptyset}\frac{\sum_{v\in L}[f(X\cup\{v\})-f(X)]}{f(X\cup L)-f(X)},\]
for $i>0$ and $\gamma_{0,j}=\gamma_{1,j}$.
\end{definition}
It can be seen that $\gamma_{i,j}$ is non-negative, non-increasing with increasing $i$ and $j$, and $f$ is submodular iff $\gamma_{i,j}\geq1$ for all $(i,j)$.

For the purpose of analysis, we denote $d=\sum_id_i$, $\bar{d}=\min_i\{d_i\}$, and $OPT$ the optimal solution; we have $\bar{d}\leq d/k$ and $|OPT|\leq d$. We evaluate the algorithm's performance via $f(X^*)/f(OPT)$ where $X^*$ is the algorithm's output. Furthermore, we use the black-box oracle model to evaluate run time, hence our results are based on numbers of oracle calls.

\subsection{Algorithms descriptions}
A popular baseline method to solve hard problems is greedy heuristics. A simple deterministic GREEDY variant has been studied for this problem \cite{GreedyMaxBoundCurvPartMatroid}. It starts with the empty solution, and in each iteration adds the feasible remaining element in $V$ that increases $f$ value the most. It terminates when there is no remaining feasible elements that yield positive gains. This algorithm extends the GREEDY algorithms in \cite{BestAlgoApproxMaxSubmodular} to partition matroids constraints. Note that at iteration $k$, GREEDY calls the oracle $n-k+1$ times, so its run time is $\mathcal{O}(dn)$. According to \cite{GreedyMaxBoundCurvPartMatroid}, it achieves $(1-e^{-\alpha \bar{d}/d})/\alpha$ approximation ratio when $f$ is submodular, and $(1-e^{-\alpha(1-\alpha)\bar{d}/d})/\alpha$ approximation ratio when $f$ is non-decreasing subadditive, with $\alpha$ being the curvature of $f$.

On the other hand, GSEMO \cite{SEMO,GSEMO,GSEMO2015}, also known as POMC \cite{POMC}, is a well-known simple Pareto optimization approach for constrained single-objective optimization problems. It has been shown to outperform the generalized greedy algorithm in overcoming local optima \cite{POMC}. To use GSEMO with partition matroid constraints, the problem is reformulated as a bi-objective problem
\[\text{maximize}_{X\subseteq V}\left(f_1(X),f_2(X)\right),\]
where $f_1(X)=\begin{cases*}
-\infty,&$\exists i,|B_i\cap X|>d_i$\\
f(X),&otherwise
\end{cases*}$, $f_2(X)=-|X|$.

\begin{algorithm}[t]
\begin{algorithmic}
\STATE \textbf{Input:} a problem instance: $(f,B,D)$
\STATE \textbf{Parameter:} the number of iterations $T\geq0$
\STATE \textbf{Output:} a feasible solution $x\in\{0,1\}^n$
\STATE $x\gets0$, $P\gets\{x\}$
\WHILE{$t<T$}
\STATE Randomly sample a solution $y$ from $P$
\STATE Generate $y'$ by flipping each bit of $y$ independently with probability $1/n$
\IF{$\nexists x\in P,x\succ y'$}
\STATE $P\gets(P\setminus\{x\in P,y'\succeq x\})\cup\{y'\}$
\ENDIF
\ENDWHILE
\STATE \textbf{return} $\argmax_{x\in P}f_1(x)$
\end{algorithmic}
\caption{GSEMO algorithm}
\label{alg:gsemo}
\end{algorithm}
GSEMO optimizes two objectives simultaneously, using the dominance relation between solutions, which is common in Pareto optimization approaches. By definition, solution $X_1$ dominates $X_2$ ($X_1\succeq X_2$) iff $f_1(X_1)\geq f_1(X_2)$ and $f_2(X_1)\geq f_2(X_2)$. The dominance relation is strict ($X_1\succ X_2$) iff $f_1(X_1)>f_1(X_2)$ or $f_2(X_1)>f_2(X_2)$. Intuitively, dominance relation formalizes the notion of ``better'' solution in multi-objective contexts. Solutions that don't dominate any other present a trade-off between objectives to be optimized.

The second objective in GSEMO is typically formulated to promote solutions that are ``further'' from being infeasible. The intuition is that for those solutions, there is more room for feasible modification, thus having more potential of becoming very good solutions. For the problem of interest, one way of measuring ``distance to infeasibility'' for some solution $X$ is counting the number of elements in $V\setminus X$ that can be added to $X$ before it is infeasible. The value then would be $d-|X|$, which is the same as $f_2(X)$ in practice. Another way is counting the minimum number of elements in $V\setminus X$ that need to be added to $X$ before it is infeasible. The value would then be $\min_i\{d_i-|B_i\cap X|\}$. The former approach is chosen for simplicity and viability under weaker assumption about the oracle.

On the other hand, the first objective aims to present the canonical evolutionary pressure based on objective values. Additionally, $f_1$ also discourages all infeasible solutions, which is different from the formulation in \cite{POMC} that allows some degree of infeasibility. This is because for $k>1$, there can be some infeasible solution $Y$ where $|Y|\leq d$. If $f_1(Y)$ is very high, it can dominate many good feasible solutions, and may prevent acceptance of global optimal solutions into the population. Furthermore, restricting to only feasible solutions decreases the maximum population size, which can improve convergence performance. It is clear the the population size is at most $d+1$. These formulations of the two objective functions are identical to the ones in \cite{GSEMO2015} when $k=1$.

In practice, set solutions are represented in GSEMO as binary sequences, where with $V=\{v_1,\dots,v_n\}$ the following bijective mapping is implicitly assumed
\[g:2^V\to\{0,1\}^n,\quad g(X)_i=\begin{cases*}
0,& $v_i\notin X$,\\
1,& $v_i\in X$
\end{cases*}.\]
This representation of set is useful in evolutionary algorithms since genetic bit operators are compatible. GSEMO operates on the bit sequences, and the fitness function is effectively a pseudo-Boolean function $f\circ g^{-1}$. It starts with initial population of a single empty solution. In each iteration, a new solution is generated by random parent selection and bit flip mutation. Then the elitist survival selection mechanism removes dominated solutions from the population, effectively maintaining a set of known Pareto-optimal solutions. The algorithm terminates when the number of iteration reaches some predetermined limit. The procedure is described in Algorithm \ref{alg:gsemo}. We choose empty set as the initial solution, similar to \cite{POMC} and different from \cite{GSEMO2015}, to simplify the analysis and stabilize theoretical performance. Note that GSEMO calls the oracle once per iteration to evaluate a new solution, so its run time is identical to the number of iterations.

\section{Approximation guarantees}
\label{sec3}
We derive an approximation guarantee for GSEMO on maximizing a general submodular function under partition matroid constraints. According to the analysis for GREEDY \cite{GreedyMaxBoundCurvPartMatroid}, we can assume there are $d$ ``dummy'' elements with zero marginal contribution. For all feasible solutions $X\subseteq V$ where $|X|<\bar{d}$, let $v^*_X=\argmax_{v\in V\setminus X}f_1(X\cup\{v\})$ be the feasible greedy addition to $X$, we can derive the following result from Lemma 2 in \cite{MaxMonoApproxSubmodMulti}, using $f(OPT\cup X)\geq f(OPT)-j\epsilon_{d+j+1}$.
\begin{lemma}[\cite{MaxMonoApproxSubmodMulti}]\label{greed_prog_submodular}
Let $f$ be a submodular function and $\epsilon_d$ be defined in Definition \ref{mono_approx}, for all feasible solutions $X\subseteq V$ such that $|X|=j<\bar{d}$
\[f(X\cup\{v^*_X\})-f(X)\geq\frac{1}{d}[f(OPT)-f(X)-j\epsilon_{d+j+1}].\]
\end{lemma}
With this lemma, we can prove the following result where $P_t$ denotes the population at iteration $t$.
\begin{theorem}\label{gsemo_submodular}
For the problem of maximizing a submodular function $f$ under partition matroid constraints, GSEMO generates in expected run time $\mathcal{O}(d^2n/k)$ a solution $X\subseteq V$ such that
\[f(X)\geq\left(1-e^{-\bar{d}/d}\right)\left[f(OPT)-(\bar{d}-1)\epsilon_{d+\bar{d}}\right].\]
\end{theorem}
\begin{proof}
Let $S(X,j)$ be a statement $|X|\leq j\wedge f(X)\geq\left[1-\left(1-\frac{1}{d}\right)^{j}\right][f(OPT)-(j-1)\epsilon_{d+j}]$, and $J_t=\max\{i\in[0,\bar{d}]|\exists X\in P_t, S(X,i)\}$, it is clear that $S(\emptyset,0)$ holds, so $J_0=0$ and $J_t$ is well-defined for all $t\geq0$ since the empty solution is never dominated.

Assuming $J_t=i$ at some $t$, let $\bar{X}\in P_t$ such that $S(\bar{X},i)$ holds. If $\bar{X}$ is not dominated and removed from $P_{t+1}$, then $J_{t+1}\geq J_{t}$. Otherwise, there must be some $Y\in P_{t+1}$ such that $|Y|\leq|\bar{X}|$ and $f(Y)\geq f(\bar{X})$. This implies $S(Y,i)$, so $J_{t+1}\geq J_t$. Therefore, $J_t$ is never decreased as $t$ progresses. Let $X'=\bar{X}\cup\{v^*_{\bar{X}}\}$, Lemma \ref{greed_prog_submodular} implies
\begin{align*}
f(X')&\geq\frac{1}{d}f(OPT)+\left(1-\frac{1}{d}\right)\left[1-\left(1-\frac{1}{d}\right)^i\right][f(OPT)-(i-1)\epsilon_{d+i}]-\frac{i}{d}\epsilon_{d+i+1}\\&
\geq\left[1-\left(1-\frac{1}{d}\right)^{i+1}\right][f(OPT)-i\epsilon_{d+i+1}].
\end{align*}
The second inequality uses $0\leq\epsilon_{d+i}\leq\epsilon_{d+i+1}$. The probability that GSEMO selects $\bar{X}$ is at least $\frac{1}{d+1}$, and the probability of generating $X'$ by mutating $\bar{X}$ is at least $\frac{1}{n}\left(1-\frac{1}{n}\right)^{n-1}\geq\frac{1}{en}$. Furthermore, $S(X',i+1)$ holds as shown, so $J_{t+1}\geq i+1$ if $X'\in P_{t+1}$. Since $i\leq\bar{d}-1$ and $t\geq0$ are chosen arbitrarily, this means
\[E[J_{t+1}-J_t|J_t\in[0,\bar{d}-1]]\geq\frac{1}{en(d+1)},\quad\forall t\geq0.\]
Therefore, the Additive Drift Theorem \cite{DriftAnalysis} implies the expected number of iterations for $J_t$ to reach $\bar{d}$ from $0$ is at most $e\bar{d}n(d+1)$. When $J_t=\bar{d}$, $P_t$ must contain a feasible solution $Z$ such that
\begin{align*}
    f(Z)&\geq\left(1-e^{-\bar{d}/d}\right)[f(OPT)-(\bar{d}-1)\epsilon_{d+\bar{d}}].
\end{align*}
Therefore, GSEMO generates such a solution in expected run time at most $e\bar{d}n(d+1)=\mathcal{O}(d^2n/k)$.
\end{proof}
In case of a single cardinality constraint ($\bar{d}=d$), this approximation guarantee is at least as tight as the one for GSEMO-C in \cite{MaxMonoApproxSubmodMulti}. If monotonicity of $f$ is further assumed, the result is equivalent to the one for GSEMO in \cite{GSEMO2015}. Additionally, the presence of $\epsilon_d$ suggests that the non-monotonicity of $f$ does not necessarily worsen the approximation guarantee when negative marginal gains are absent from all GSEMO's possible solutions (i.e. cannot decrease objective values by adding an element).

As an extension beyond submodularity instead of monotonicity, we provide another proof of the approximation guarantee for GSEMO on the problems of maximizing monotone functions under the same constraints. Without loss of generality, we can assume that $f$ is normalized, meaning $f(\emptyset)=0$. We make use of the following inequality, derived from Lemma 1 in \cite{PPOSS}.
\begin{lemma}\label{greed_prog_monotone}
Let $f$ be a monotone function and $\gamma_{i,j}$ be defined in Definition \ref{submodratio_def}, for all feasible solutions $X\subseteq V$ such that $|X|=j<\bar{d}$
\[f(X\cup\{v^*_X\})-f(X)\geq\frac{\gamma_{j+1,d}}{d}[f(OPT)-f(X)].\]
\end{lemma}
Using this lemma, we similarly prove the following result.
\begin{theorem}\label{gsemo_monotone}
For the problem of maximizing a monotone function under partition matroid constraints, GSEMO with expected run time $\mathcal{O}(d^2n/k)$ generates a solution $X\subseteq V$ such that
\[f(X)\geq\left(1-e^{-\gamma_{\bar{d},d}\bar{d}/d}\right)f(OPT).\]
\end{theorem}
\begin{proof}
Let $S(X,j)$ be a statement $|X|\leq j\wedge f(X)\geq\left[1-\left(1-\frac{\gamma_{j,d}\bar{d}}{d}\right)^{j}\right]f(OPT)$, and $J_t=\max\{i\in[0,\bar{d}]|\exists X\in P_t, S(X,i)\}$, it is clear that $S(\emptyset,0)$ holds, so $J_0=0$ and $J_t$ is well-defined for all $t\geq0$ since the empty solution is never dominated.

Assuming $J_t=i$ at some $t$, there must be $\bar{X}\in P_t$ such that $S(\bar{X},i)$ holds. If $\bar{X}$ is not dominated and removed from $P_{t+1}$, then $J_{t+1}\geq J_{t}$. Otherwise, there must be some $Y\in P_{t+1}$ such that $|Y|\leq|\bar{X}|$ and $f(Y)\geq f(\bar{X})$. This implies $S(Y,i)$, so $J_{t+1}\geq J_t$. Therefore, $J_t$ is never decreased as $t$ progresses. Let $X'=\bar{X}\cup\{v^*_{\bar{X}}\}$, Lemma \ref{greed_prog_monotone} implies
\begin{align*}
f(X')&\geq\frac{\gamma_{i+1,d}}{d}f(OPT)+\left(1-\frac{\gamma_{i+1,d}}{d}\right)\left[1-\left(1-\frac{\gamma_{i,d}}{d}\right)^i\right]f(OPT)\\&
\geq\left[1-\left(1-\frac{\gamma_{i+1,d}}{d}\right)^{i+1}\right]f(OPT).
\end{align*}
The second inequality uses $\gamma_{i,d}\geq\gamma_{i+1,d}$.The probability that GSEMO selects $\bar{X}$ is at least $\frac{1}{d+1}$, and the probability of generating $X'$ by mutating $\bar{X}$ is at least $\frac{1}{n}\left(1-\frac{1}{n}\right)^{n-1}\geq\frac{1}{en}$. Furthermore, $S(X',i+1)$ holds as shown, so $J_{t+1}\geq i+1$. Since $i\leq\bar{d}-1$ and $t\geq0$ are chosen arbitrarily, this means
\[E[J_{t+1}-J_t|J_t\in[0,\bar{d}-1]]\geq\frac{1}{en(d+1)},\quad\forall t\geq0.\]
Therefore, according to the Additive Drift Theorem \cite{DriftAnalysis}, the expected number of iterations for $J_t$ to reach $\bar{d}$ from $0$ is at most $e\bar{d}n(d+1)$. When $J_t=\bar{d}$, $P_t$ must contain a feasible solution $Z$ such that
\[f(Z)\geq\left[1-\left(1-\frac{\gamma_{\bar{d},d}}{d}\right)^{\bar{d}}\right]f(OPT)\geq\left(1-e^{-\gamma_{\bar{d},d}\bar{d}/d}\right)f(OPT).\]
Therefore, GSEMO generates such a solution in expected run time at most $e\bar{d}n(d+1)=\mathcal{O}(d^2n/k)$.
\end{proof}
Compared to the results in \cite{POMC}, it is reasonable to assume that restricting GSEMO's population to only allow feasible solutions improves worst-case guarantees. However, it also eliminates the possibility of efficient improvement by modifying infeasible solutions that are very close to very good feasible ones. This might reduce its capacity to overcome local optima.

\section{Experimental investigations}
\label{sec4}
We compare GSEMO and GREEDY on the symmetric submodular Cut maximization problems with randomly generated graphs under varying settings. The experiments are separated into two groups: cardinality constraints ($k=1$) and general partition matroid constraints ($k>1$). 

\subsection{Max Cut problems setup}

Weighted graphs are generated for the experiments based on two parameters: number of vertices (which is $n$) and density. There are 3 values for $n$: 50, 100, 200. There are 5 density values: 0.01, 0.02, 0.05, 0.1, 0.2. For each $n$-$density$ pair, 30 different weighted graphs -- each denoted as $G=(V,E,c)$ -- are generated with the following procedure:
\begin{enumerate}
    \item Randomly sample $E$ from $V\times V$ without replacement, until $|E|=\lfloor density\times n^2\rfloor$.
    \item Assign to $c(a,b)$ a uniformly random value in $[0,1]$ for each $(a,b)\in E$.
    \item Assign $c(a,b)=0$ for all $(a,b)\notin E$.
\end{enumerate}
Each graph is then paired with  different sets of constraints, and each pairing constitutes a problem instance. This enables observations of changes in outputs on the same graphs under varying constraints. For cardinality constraints, $d_1=\{\frac{n}{4},\frac{n}{2},\frac{3n}{4}\}$, rounded to the nearest integer. Thus, there are 30 problem instances per $n$-density-$d_1$ triplet. For partition matroid constraints, the numbers of partitions are $k=\{2,5,10\}$. The partitions are of the same size, and each element is randomly assigned to a partition. The thresholds $d_i$ are all set to $\lceil\frac{n}{2k}\rceil$ since the objective functions are symmetric. Likewise, there are 30 problem instances per $n$-density-$k$ triplet.

GSEMO is run on each instance 30 times, and the minimum, mean and maximum results are denoted by GSEMO$^-$, GSEMO$^*$ and GSEMO$^+$, respectively. The GREEDY algorithm is run until satisfying its stopping condition, while GSEMO is run for $T=4n^2$ iterations. Their final achieved objective values are then recorded and analyzed. Note that the run time budget for GSEMO in every setting is smaller than the theoretical bound on average run time in Theorem \ref{gsemo_submodular}, except for $(n,k)=(50,10)$ where it is only slightly larger.
\subsection{Cut maximization under a cardinality constraint}

\begin{table*}[htp!]
\begin{normalsize}
\renewcommand{\arraystretch}{.96}
\caption{Experimental results for cardinality constraints cases. Ranges of final objective values across 30 graphs are shown for each setting. The signed-rank U-tests are used to compare GREEDY's with GSEMO$^-$, GSEMO$^*$, GSEMO$^+$ for each setting, pairing by instances, with 95\% confidence level. `+' denotes values being significantly greater than GREEDY's, `-' denotes less, and `*' denotes no significant difference. Additionally, numbers of losses, wins and ties (L-W-T) GSEMO has over GREEDY are shown, which are determined by separate U-tests on individual instances.}
\label{table:experiments1}
\begin{adjustbox}{tabular=lllcccccccccccc,center}
\toprule
\multirow{2}{*}{$n$}&\multirow{2}{*}{density}&\multirow{2}{*}{$d_1$}&\multicolumn{2}{c}{GREEDY}&\multicolumn{3}{c}{GSEMO$^-$}&\multicolumn{3}{c}{GSEMO$^*$}&\multicolumn{3}{c}{GSEMO$^+$}&\multirow{2}{*}{L--W--T}\\\cmidrule(l{2pt}r{2pt}){4-5}\cmidrule(l{2pt}r{2pt}){6-8}\cmidrule(l{2pt}r{2pt}){9-11}\cmidrule(l{2pt}r{2pt}){12-14}&&&min&max&min&max&stat&min&max&stat&min&max&stat&\\\cmidrule(l{2pt}r{2pt}){1-15}
\multirow{15}{*}{\vspace{-38pt}50}&\multirow{3}{*}{0.01}&13&6.638&12.71&6.625&12.59&-&6.635&12.7&-&6.638&12.71&*&3--3--24\\&
&25&6.638&13.27&6.625&13.13&-&6.65&13.25&*&6.706&13.27&+&2--7--21\\&
&38&6.638&13.27&6.625&13.13&-&6.647&13.24&*&6.706&13.27&+&2--6--22\\\cmidrule(l{2pt}r{2pt}){2-15}
&\multirow{3}{*}{0.02}&13&12.16&18.08&12.1&18.08&-&12.15&18.08&*&12.17&18.08&+&6--5--19\\&
&25&13.5&20.27&13.43&20.17&-&13.47&20.29&+&13.5&20.33&+&8--16--6\\&
&38&13.5&20.27&13.39&20.23&-&13.47&20.3&+&13.5&20.33&+&7--17--6\\\cmidrule(l{2pt}r{2pt}){2-15}
&\multirow{3}{*}{0.05}&13&22.09&29.38&21.69&29.28&-&22.03&29.49&+&22.09&29.51&+&2--17--11\\&
&25&26.52&37.14&27.18&36.28&-&27.75&36.95&+&28.1&37.14&+&4--20--6\\&
&38&26.52&37.14&27.39&36.43&-&27.89&37.02&+&28.2&37.14&+&5--21--4\\\cmidrule(l{2pt}r{2pt}){2-15}
&\multirow{3}{*}{0.1}&13&38.64&47.14&38.19&46.82&-&38.59&47.12&*&38.69&47.14&+&5--11--14\\&
&25&46.61&57.93&45.36&57.38&-&46.85&57.85&+&47.28&58.03&+&11--16--3\\&
&38&46.61&58.08&45.03&57.58&-&46.77&57.95&+&47.28&58.1&+&10--17--3\\\cmidrule(l{2pt}r{2pt}){2-15}
&\multirow{3}{*}{0.2}&13&63.13&77.38&63.09&77.01&*&63.26&77.43&+&63.46&77.64&+&2--16--12\\&
&25&78.89&92.57&79.37&91.68&-&80.61&92.25&+&80.82&92.66&+&5--21--4\\&
&38&78.89&92.57&79.82&91.78&-&80.62&92.31&+&80.82&92.66&+&5--21--4\\\cmidrule(l{2pt}r{2pt}){1-15}
\multirow{15}{*}{\vspace{-37pt}100}&\multirow{3}{*}{0.01}&25&24.93&31.88&24.9&31.86&-&25.11&31.89&*&25.36&31.89&+&7--9--14\\&
&50&27.87&37.79&27.56&37.87&-&28.11&38.07&*&28.67&38.26&+&12--13--5\\&
&75&27.87&37.79&27.24&37.71&-&28.07&38&+&28.67&38.26&+&11--13--6\\\cmidrule(l{2pt}r{2pt}){2-15}
&\multirow{3}{*}{0.02}&25&42.95&53.66&42.95&53.56&-&43.38&53.64&*&43.4&53.66&+&7--10--13\\&
&50&51.93&66&51.95&65.27&-&52.62&66.14&+&52.78&66.4&+&9--15--6\\&
&75&51.93&66&51.69&64.31&-&52.6&66.08&+&52.8&66.4&+&9--17--4\\\cmidrule(l{2pt}r{2pt}){2-15}
&\multirow{3}{*}{0.05}&25&78.13&94.98&78.09&95.3&-&78.38&95.45&+&78.78&95.49&+&8--17--5\\&
&50&100.6&120.7&100&119&-&100.9&120.3&+&101.7&120.7&+&7--18--5\\&
&75&100.7&120.7&99.23&118.9&-&100.8&120.4&+&101.9&120.7&+&7--17--6\\\cmidrule(l{2pt}r{2pt}){2-15}
&\multirow{3}{*}{0.1}&25&138.9&155.3&138.5&155&-&139.1&156&+&139.6&156.2&+&2--20--8\\&
&50&178.4&197.8&177&198.3&*&178&199.9&+&179.6&200.6&+&7--17--6\\&
&75&178.4&197.8&176.6&198.6&*&178&199.9&+&179.4&200.6&+&6--17--7\\\cmidrule(l{2pt}r{2pt}){2-15}
&\multirow{3}{*}{0.2}&25&224.1&249.2&222.8&249.4&*&224&250&+&225.6&250.2&+&4--22--4\\&
&50&297.6&325.1&297.8&323.9&*&300.9&325.8&+&302.7&326.4&+&6--20--4\\&
&75&297.6&325.1&298&323.9&-&300.4&325.8&+&303.2&326.4&+&6--19--5\\\cmidrule(l{2pt}r{2pt}){1-15}
\multirow{15}{*}{\vspace{-37pt}200}&\multirow{3}{*}{0.01}&50&85.54&96.11&84.98&96.05&*&85.58&96.3&+&85.84&96.52&+&7--20--3\\&
&100&103.1&118.4&103.5&118.5&-&104.6&120.1&+&104.9&121.4&+&4--23--3\\&
&150&103.1&118.4&103.6&118.6&-&104.7&120&+&105.1&121.4&+&4--21--5\\\cmidrule(l{2pt}r{2pt}){2-15}
&\multirow{3}{*}{0.02}&50&139.1&159.3&140.6&158.8&*&141.8&159.2&+&142.2&159.3&+&8--19--3\\&
&100&173.3&198.2&175.6&197.5&*&177.3&198.5&+&179.4&199.3&+&2--25--3\\&
&150&173.3&198.2&174.8&197.2&*&177&198.4&+&178.9&199.5&+&2--23--5\\\cmidrule(l{2pt}r{2pt}){2-15}
&\multirow{3}{*}{0.05}&50&275.9&311.8&277.8&311.4&+&278.8&312.3&+&280&313&+&0--28--2\\&
&100&357&400.6&364.4&402.6&*&367.1&405&+&369.7&406.7&+&1--28--1\\&
&150&357&400.6&364.2&402.7&*&366.8&405.2&+&369.9&407.1&+&0--27--3\\\cmidrule(l{2pt}r{2pt}){2-15}
&\multirow{3}{*}{0.1}&50&489.8&534&490.6&533.7&*&492.6&534&+&493.4&534&+&6--22--2\\&
&100&647.8&680.5&643.7&679.3&-&649.5&683.7&+&653&686.4&+&2--24--4\\&
&150&648&680.5&642.3&680.9&*&649.6&683.8&+&652.2&687&+&2--22--6\\\cmidrule(l{2pt}r{2pt}){2-15}
&\multirow{3}{*}{0.2}&50&866.9&921.8&867.9&920.6&*&871.5&921.6&+&873.4&921.8&+&1--26--3\\&
&100&1120&1182&1120&1181&+&1125&1188&+&1128&1192&+&0--29--1\\&
&150&1120&1182&1122&1181&+&1125&1189&+&1129&1193&+&0--30--0\\\cmidrule(l{2pt}r{2pt}){1-15}
\end{adjustbox}
\end{normalsize}
\end{table*}

\begin{table*}[htp!]
\begin{normalsize}
\renewcommand{\arraystretch}{.96}
\centering
\caption{Experimental results for partition matroid constraints cases. Ranges of final objective values across 30 graphs are shown for each setting. The signed-rank U-tests are used to compare GREEDY's with GSEMO$^-$, GSEMO$^*$, GSEMO$^+$ for each setting, pairing by instances, with 95\% confidence level. `+' denotes values being significantly greater than GREEDY's, `-' denotes less, and `*' denotes no significant difference. Additionally, numbers of losses, wins and ties (L-W-T) GSEMO has over GREEDY are shown, which are determined by separate U-tests on individual instances.}
\label{table:experiments11}
\begin{adjustbox}{tabular=lllcccccccccccc,center}
\toprule
\multirow{2}{*}{$n$}&\multirow{2}{*}{density}&\multirow{2}{*}{$k$}&\multicolumn{2}{c}{GREEDY}&\multicolumn{3}{c}{GSEMO$^-$}&\multicolumn{3}{c}{GSEMO$^*$}&\multicolumn{3}{c}{GSEMO$^+$}&\multirow{2}{*}{L--W--T}\\\cmidrule(l{2pt}r{2pt}){4-5}\cmidrule(l{2pt}r{2pt}){6-8}\cmidrule(l{2pt}r{2pt}){9-11}\cmidrule(l{2pt}r{2pt}){12-14}&&&min&max&min&max&stat&min&max&stat&min&max&stat&\\\cmidrule(l{2pt}r{2pt}){1-15}
\multirow{15}{*}{\vspace{-38pt}50}&\multirow{3}{*}{0.01}&2&6.638&13.27&6.625&13.15&-&6.65&13.26&*&6.706&13.27&+&1--7--22\\&
&5&6.638&13.27&6.625&13.13&-&6.654&13.25&*&6.706&13.27&+&3--7--20\\&
&10&6.638&13.27&6.625&13.15&-&6.65&13.25&*&6.706&13.27&+&1--7--22\\\cmidrule(l{2pt}r{2pt}){2-15}
&\multirow{3}{*}{0.02}&2&13.5&20.27&13.32&20.24&-&13.47&20.28&+&13.5&20.33&+&7--16--7\\&
&5&13.5&19.69&13.14&19.53&-&13.4&19.65&+&13.5&19.69&+&7--13--10\\&
&10&13.11&20.27&12.94&20.07&-&13.09&20.28&+&13.11&20.33&+&7--14--9\\\cmidrule(l{2pt}r{2pt}){2-15}
&\multirow{3}{*}{0.05}&2&25.8&37.14&26.63&36.82&-&27.07&37.07&+&27.42&37.14&+&5--18--7\\&
&5&25.78&36.75&26.49&35.86&-&27.04&36.58&+&27.66&36.88&+&8--18--4\\&
&10&25.9&36.83&25.73&35.28&-&27.38&36.17&+&28.09&36.83&+&6--20--4\\\cmidrule(l{2pt}r{2pt}){2-15}
&\multirow{3}{*}{0.1}&2&46.61&57.93&44.87&56.97&-&46.72&57.7&+&47.28&58.05&+&9--18--3\\&
&5&45.91&56.69&45.3&55.85&-&46.23&56.4&+&46.81&56.94&+&8--18--4\\&
&10&46.59&56.46&44.13&55.73&-&46.41&56.6&+&47.08&57.13&+&6--18--6\\\cmidrule(l{2pt}r{2pt}){2-15}
&\multirow{3}{*}{0.2}&2&78.89&92.5&79.16&91.53&-&80.54&92.18&+&80.82&92.66&+&6--19--5\\&
&5&74.37&91.89&74.96&90.83&-&76.44&91.61&+&77.39&91.89&+&6--22--2\\&
&10&75.85&92.57&75.62&90.86&-&76.7&92.08&+&77.97&92.51&+&4--21--5\\\cmidrule(l{2pt}r{2pt}){1-15}
\multirow{15}{*}{\vspace{-37pt}100}&\multirow{3}{*}{0.01}&2&27.87&37.79&27.69&37.7&-&28.15&37.99&+&28.67&38.26&+&10--14--6\\&
&5&27.87&37.79&27.78&37.65&-&28.1&38&+&28.67&38.26&+&10--14--6\\&
&10&27.87&36.81&27.59&36.51&-&27.88&36.91&*&28.31&37.14&+&11--12--7\\\cmidrule(l{2pt}r{2pt}){2-15}
&\multirow{3}{*}{0.02}&2&51.92&66&51.64&65.44&-&52.55&66.17&+&52.79&66.4&+&10--15--5\\&
&5&51.62&65.82&51.57&65.48&-&52.41&66.04&*&52.71&66.22&+&8--12--10\\&
&10&51.69&63.41&51.38&62.63&-&51.9&63.29&*&52.19&63.75&+&8--13--9\\\cmidrule(l{2pt}r{2pt}){2-15}
&\multirow{3}{*}{0.05}&2&100.1&120.7&99.68&119.3&-&100.9&120.2&+&101.7&120.7&+&8--17--5\\&
&5&99.61&119.5&98.19&117.4&-&99.86&119&+&100.6&119.7&+&7--17--6\\&
&10&97&116.6&95.82&116&-&98.14&117.6&+&99.37&118.3&+&5--18--7\\\cmidrule(l{2pt}r{2pt}){2-15}
&\multirow{3}{*}{0.1}&2&177.6&197.8&176.8&198.6&*&177.6&199.9&+&178.4&200.7&+&6--19--5\\&
&5&173.5&196.6&172.6&197.6&-&174.3&199.1&+&175.8&200.2&+&2--21--7\\&
&10&173.3&192.9&171.7&193.9&-&173.6&195.9&*&175.3&198.2&+&12--12--6\\\cmidrule(l{2pt}r{2pt}){2-15}
&\multirow{3}{*}{0.2}&2&294.2&325.1&294.8&324.8&*&297.4&325.9&+&301.6&326.4&+&6--22--2\\&
&5&292.7&320.5&293.2&318.5&*&297.3&321.8&+&299.5&323.7&+&3--23--4\\&
&10&288.3&322.7&288.2&317.3&-&292.3&321.6&+&296.3&323.8&+&5--22--3\\\cmidrule(l{2pt}r{2pt}){1-15}
\multirow{15}{*}{\vspace{-37pt}200}&\multirow{3}{*}{0.01}&2&103.1&118.4&103.6&118.9&*&104.6&120.2&+&105.1&121.6&+&4--22--4\\&
&5&103.1&118.4&103.4&118.5&-&104.2&119.9&+&104.7&121.4&+&3--20--7\\&
&10&102.7&117.4&102.6&117.2&-&104&118.1&+&104.7&119.3&+&4--21--5\\\cmidrule(l{2pt}r{2pt}){2-15}
&\multirow{3}{*}{0.02}&2&173.3&198.2&174.8&197.2&*&176.5&198.5&+&178.6&199.5&+&3--22--5\\&
&5&172.8&196.7&173.7&195.7&-&176.4&197.6&+&179.1&198.6&+&2--20--8\\&
&10&172.5&193.3&173.7&193.9&*&176.1&195.6&+&177.9&196.8&+&1--23--6\\\cmidrule(l{2pt}r{2pt}){2-15}
&\multirow{3}{*}{0.05}&2&356.4&400.6&362.7&401.6&*&366.2&404.8&+&369.1&406.1&+&2--28--0\\&
&5&353.7&399.4&359.6&400.5&*&363.5&403&+&365.5&404.7&+&1--28--1\\&
&10&352.9&394.3&355.1&396.7&*&360.1&399.4&+&363.2&401.5&+&1--26--3\\\cmidrule(l{2pt}r{2pt}){2-15}
&\multirow{3}{*}{0.1}&2&645.2&680.5&642.8&679&*&647.4&682.9&+&650.7&685.5&+&3--25--2\\&
&5&641.7&678.7&637.4&676.2&-&643.3&680.6&+&647.4&685.2&+&4--20--6\\&
&10&637.2&669.8&632.3&667.4&-&641.5&672.7&+&645.1&677.3&+&2--24--4\\\cmidrule(l{2pt}r{2pt}){2-15}
&\multirow{3}{*}{0.2}&2&1119&1182&1118&1183&+&1123&1187&+&1128&1193&+&0--30--0\\&
&5&1117&1178&1116&1173&*&1121&1179&+&1125&1184&+&0--29--1\\&
&10&1105&1170&1111&1175&*&1117&1181&+&1122&1188&+&2--28--0\\\cmidrule(l{2pt}r{2pt}){1-15}
\end{adjustbox}
\end{normalsize}
\end{table*}

The experimental results for cardinality constraint cases are shown in Table \ref{table:experiments1}. Signed-rank U-tests \cite{Corder09} are applied to the outputs, with pairing based on instances. Furthermore, we count the numbers of instances where GSEMO outperforms, ties with, and is outperformed by GREEDY via separate U-tests on individual instances.

Overall, GSEMO on average outperforms GREEDY with statistical significance in most cases. The significance seems to increase, with some noises, with increasing $n$, density, and $d_1$. This indicates that GSEMO more reliably produces better solutions than GREEDY as the graph's size and density increase. Moreover, in few cases with large $n$, GSEMO$^-$ is higher than GREEDY's with statistical significance.

Additionally, it is indicated that GSEMO$^*$ tend to be closer to GSEMO$^+$ than GSEMO$^-$. This suggests skewed distribution of outputs in each instance toward high values. The implication is that GSEMO is more likely to produce outputs greater than average, than otherwise. It might be an indication that these results are close to the global optima for these instances.

Per instance analyses show high number of ties between GSEMO and GREEDY for small $n$, density, and to a lesser extent $d_1$. As these increase, the number of GSEMO's wins increases and ends up dominating at $n=200$. This trend coincides with earlier observations, and suggests the difficulty of making improvements in sparse graphs faced by GSEMO where GREEDY heuristic seems more suited. On the other hand, large graph sizes seem to favour GSEMO over GREEDY despite high sparsity, likely due to more local optima present in larger search spaces.

\subsection{Cut maximization under partition matroid constraints}
The experimental results for partition matroid constraints cases are shown in Table \ref{table:experiments11}. Notations and statistical test procedure are the same as in Table \ref{table:experiments1}.

Overall, the main trend in cardinality constraint cases is present: GSEMO on average outperforms GREEDY, with increased reliability at larger $n$ and density. This can be observed in both the average performances and the frequency at which GSEMO beats GREEDY. It seems the effect of this phenomenon is less influenced by variations in $k$ than it is by variations in $d_1$ in cardinality constraint cases. Note that the analysis in Theorem \ref{gsemo_submodular} only considers bottom-up improvements up to $|X|=\bar{d}$. Experimental results likely suggest GSEMO can make similar improvements beyond that point up to $|X|=d$.

Additionally, the outputs of both algorithms are generally decreased with increased $k$, due to restricted feasible solution spaces. There are few exceptions which could be attributed to noises from random partitioning since they occur in both GREEDY's and GSEMO's simultaneously. Furthermore, the variation in $k$ seems to slightly affect the gaps between GSEMO's and GREEDY's, which are somewhat smaller at higher $k$. It seems to support the notion that GREEDY performs well in small search spaces while GSEMO excels in large search spaces. This coincides with the observations in the cardinality constraint cases, and explains the main trend.

\section{Conclusion}
In this work, we consider the problem of maximizing a set function under partition matroid constraints, and analyze GSEMO's approximation performance on such problems. Theoretical performance guarantees are derived for GSEMO in cases of submodular objective functions, and monotone objective functions. We show that GSEMO guarantees good approximation quality within polynomial expected run time in both cases. Additionally, experiments with Max Cut instances generated from varying settings have been conducted to gain insight on its empirical performance, based on comparison against simple GREEDY's. The results show that GSEMO generally outperforms GREEDY within quadratic run time, particularly when the feasible solution space is large.

\section*{Acknowledgements} 
The experiments were run using the HPC service provided by the University of Adelaide.
%
%

%
%
%
\bibliographystyle{unsrt}
\bibliography{ref}
\end{document}